\documentclass[11pt]{article}

\usepackage[margin=1in]{geometry}
\usepackage{amsmath, amsthm, amssymb, mathrsfs}
\usepackage{color}
\usepackage{natbib}
\usepackage{algorithm}
\usepackage{algorithmic}
\usepackage{hyperref}

\newtheorem{theorem}{Theorem}[section]
\newtheorem{lemma}[theorem]{Lemma}

\newtheorem{assumption}{Assumption}
\newtheorem{remark}{Remark}

\newcommand{\R}{\mathbb{R}}
\newcommand{\E}{\mathbb{E}}
\newcommand{\Pro}{\mathbb{P}}
\newcommand{\X}{\mathbf{X}}
\newcommand{\x}{\mathbf{x}}
\newcommand{\PsiVec}{\mathbf{\Psi}}
\newcommand{\thetaVec}{\mathbf{\theta}}
\newcommand{\I}{\mathbb{I}}

\newcommand{\B}{\mathcal{B}_1}
\newcommand{\inner}[2]{\langle #1, #2 \rangle}

\allowdisplaybreaks

\title{Online Quantile Regression  for Nonparametric \\ Additive Models }
\author{Haoran Zhan \thanks{haoran.zhan@u.nus.edu}\\
    Department of Statistics and Data Science, \\
  Southwestern University of Finance and Economics}
\date{April 10, 2026}

\begin{document}

\maketitle

\begin{abstract}
This paper introduces a projected functional gradient descent algorithm (P-FGD) for training nonparametric additive quantile regression models in  online settings. This algorithm extends the functional stochastic gradient descent  framework to the pinball loss. An advantage of P-FGD is that it does not need to store historical data while maintaining $O(J_t\ln J_t)$ computational complexity per step where $J_t$ denotes the number of basis functions. Besides, we only need $O(J_t)$ computational time for quantile function prediction at time $t$. These properties show that P-FGD is much better than the commonly used RKHS in online learning. By leveraging a novel Hilbert space projection identity, we also prove that the proposed online quantile function estimator (P-FGD) achieves the minimax optimal consistency rate $O(t^{-\frac{2s}{2s+1}})$ where   $t$ is the current time and $s$ denotes the  smoothness degree of the quantile function. Extensions to mini-batch learning are also established.\\

\end{abstract}

\section{Introduction}
Extensive literature has focused on online least squares estimation for the conditional mean \citep{zhang2022sieve} but real-world data frequently exhibit heteroscedasticity and heavy-tailed noise. Quantile regression \citep{koenker1978regression} offers a robust alternative by estimating the conditional $\tau$-quantile function $q_\tau(\x)$, defined such that 
$$
q_\tau(\x):= \inf_{v\in\R}\{v: F_{Y|\X=\x}(v)\ge \tau\},
$$
where $(\X,Y)\in[0,1]^p\times\R$ is a pair of random vectors and  $F_{Y|\X=\x}$ denotes the conditional distribution function. On the other hand, the target function $q_\tau$ minimizes the expected pinball loss:
\begin{equation}
    \min_{q} \E \left[ \rho_\tau(Y - q(\X)) \right], \quad \text{where } \rho_\tau(u) = u(\tau - \I(u \le 0)).
\end{equation}
Therefore, the loss function $\rho_\tau$ is frequently used in literature to estimate $q_\tau$.

Until now, online quantile regression has not gained much attention in literature and  current papers only consider  linear models, see \cite{shen2025online}.  In this paper, we are interested in studying nonparametric additive models. Nonparametric additive models \citep{hastie1990generalized} provide a flexible alternative to linear models, mitigating the curse of dimensionality while retaining structural interpretability. Unlike  many previous online learning papers, here we aim to construct an online estimator satisfying  an important property that 
\begin{center}
{\textbf{historical data can not be visited during its construction}.} 
\end{center}
This property ensures that our online algorithm still works for big data; otherwise, machines do not have space to store such big data and also do not have abilities to do  numerical analysis related to such estimator. Thus, RKHS is not suitable in this case (e.g. \cite{zhang2022sieve} ) and gradient descent  method is instead a wise choice. 

However, adapting  Gradient Descent (GD) to nonparametric quantile regression presents severe theoretical hurdles. Unlike the squared loss, the pinball loss is not globally strongly convex; its curvature depends entirely on the unknown conditional density $p_{Y|\X}$. If the estimator diverges, the density evaluations approach zero, neutralizing the descent direction. In this paper, we propose a projected Functional GD (P-FGD) estimator that overcomes this bottleneck.

\subsection{Contributions}
\begin{itemize}
    \item \textbf{Algorithmic Innovation:} We propose a functionally motivated online quantile estimator called P-FGD. To make the estimator be well-defined, we apply an exact $\ell_1$-ball projection to the basis coefficients. This requires only $O(J_t\ln J_t)$ time, preserving the computational efficiency of standard F-GD. Meanwhile, our method \textbf{does not} need storing any historical data.
    \item \textbf{Theoretical Resolution:} We introduce a novel Hilbert space orthogonal decomposition to analyze the expected cross-term generated by the non-linear subgradient. We establish that our estimator achieves the minimax optimal rate $O(t^{-2s/(2s+1)})$ for any smooth function with smoothness degree $s$, where $t$ denotes the current time.
    \item \textbf{Mini-Batch Extension:} We prove that our algorithm seamlessly handles mini-batch updates, maintaining the minimax optimal rate with respect to the total cumulative sample size.
\end{itemize}

\section{Preliminaries and Assumptions}
\subsection{Function Spaces}\label{sec:basis}

Let $L_2$ represent the collection of univariate square-integrable functions. For notational convenience, we omit the explicit domain initially, though we will subsequently restrict our focus to the unit interval $[0,1]$. A sequence of functions $\{\psi_j\}_{j=1}^\infty \subset L_2$ forms an orthonormal basis with respect to the Lebesgue measure provided that
\begin{equation*}
    \int \psi_i(x)\psi_j(x)dx = \delta_{ij},
\end{equation*}
where $\delta_{ij}$ denotes the Kronecker delta. Furthermore, we define this basis to be:
\begin{enumerate}
    \item[(i)] \textit{centered} if 
    \begin{equation*}
        \int \psi_j(x)dx = 0, \quad \text{for all } j \ge 1;
    \end{equation*}
    \item[(ii)] \textit{complete} if every function $f \in L_2$ can be uniquely expressed as
    \begin{equation*}
        f = \sum_{j=1}^\infty \theta_{f,j}\psi_j,
    \end{equation*}
    for some sequence of coefficients $\{\theta_{f,j}\}_{j=1}^\infty \in \ell^2$ (the space of square-summable sequences).
\end{enumerate}

A standard example of a complete orthonormal basis on $L_2[0,1]$ is the trigonometric basis, given by $\psi_0(x) = 1$, $\psi_{2k-1}(x) = \sqrt{2}\sin(2\pi k x)$, and $\psi_{2k}(x) = \sqrt{2}\cos(2\pi k x)$ for $k \ge 1$. Dropping the constant term $\psi_0$ leaves a centered orthonormal basis. Wavelet bases and Legendre bases are other common choices in the literature \citep{tsybakov2009introduction}.

For a given, though not strictly complete, orthonormal basis $\{\psi_j\}$, let $\mathcal{F}_1(\{\psi_j\})$ denote the class of univariate functions that can be expanded in this basis. Formally,
\begin{equation*}
    \mathcal{F}_1(\{\psi_j\}) := \left\{ f = \sum_{j=1}^\infty \theta_{f,j}\psi_j : \sum_{j=1}^\infty \theta_{f,j}^2 < \infty \right\}.
\end{equation*}
If the basis is complete, $\mathcal{F}_1(\{\psi_j\})$ trivially coincides with $\mathcal{L}^2$. We next introduce the univariate Sobolev space $\mathcal{H}_1(s, \{\psi_j\})$ associated with a smoothness parameter $s$:
\begin{equation*}
    \mathcal{H}_1(s, \{\psi_j\}) := \left\{ f = \sum_{j=1}^\infty \theta_{f,j}\psi_j : \sum_{j=1}^\infty (j^s \theta_{f,j})^2 < \infty \right\}.
\end{equation*}
Similarly, the Sobolev ellipsoid of radius $Q$ is defined as:
\begin{equation*}
    \mathcal{W}_1(s, Q, \{\psi_j\}) := \left\{ f = \sum_{j=1}^\infty \theta_{f,j}\psi_j : \sum_{j=1}^\infty (j^s \theta_{f,j})^2 < Q^2 \right\}.
\end{equation*}

Clearly, $\mathcal{W}_1(s, Q, \{\psi_j\}) \subset \mathcal{H}_1(s, \{\psi_j\})$. For any $f \in \mathcal{H}_1(s, \{\psi_j\})$, its Sobolev norm is given by $\|f\|_s^2 := \sum_{j=1}^\infty (j^s \theta_{f,j})^2$. The parameter $s > 0$ dictates the smoothness level of the functions. For the theoretical derivations in this paper, we treat $s$ as a known constant. Such Sobolev classes are foundational in  nonparametric regression analysis \citep{gyorfi2002distribution, tsybakov2009introduction}.

To accommodate additive modeling, we extend these definitions to the multivariate setting. Consider $p$ centered orthonormal bases $\{\psi_{1j}\}, \{\psi_{2j}\}, \dots, \{\psi_{pj}\}$. We define the multivariate additive function space as:
\begin{align*}
    &\mathcal{F}_p(\{\psi_{1j}\}, \{\psi_{2j}\}, \dots, \{\psi_{pj}\}) \\
    &:= \left\{ f: \mathbb{R}^p \to \mathbb{R} \mid f = \alpha + \sum_{k=1}^p f_k, \text{ where } \alpha \in \mathbb{R}, \; f_k \in \mathcal{F}_1(\{\psi_{kj}\}), \; k = 1, 2, \dots, p \right\}.
\end{align*}
If we utilize the centered trigonometric basis for each dimension, this space simply becomes the direct sum of $p$ individual $L_2$ spaces along with an intercept:
\begin{equation*}
    \mathcal{F}_p(\{\psi_{1j}\}, \{\psi_{2j}\}, \dots, \{\psi_{pj}\}) = \left\{ f = \alpha + f_1 + f_2 + \dots + f_p : f_k \in \mathcal{L}^2, \text{ where } k = 1, 2, \dots, p \right\}.
\end{equation*}

Following the same logic, we construct the additive Sobolev space $\mathcal{H}_p(s, \{\psi_{1j}\}, \dots, \{\psi_{pj}\})$ and the additive Sobolev ellipsoid $\mathcal{W}_p(s, Q, \{\psi_{1j}\}, \dots, \{\psi_{pj}\})$ by requiring each univariate component $f_k$ to reside in $\mathcal{H}_1(s, \{\psi_{kj}\})$ and $\mathcal{W}_1(s, Q, \{\psi_{kj}\})$, respectively. While our methodology readily adapts to varying smoothness levels across dimensions, we assume a uniform smoothness $s$ for all components to simplify the presentation. When the chosen bases are unambiguous, we will use the abbreviated notations $\mathcal{F}_p$, $\mathcal{H}_p(s)$, and $\mathcal{W}_p(s, Q)$.

Suppose $q_\tau\in \mathcal{W}_p(s, Q)$ and $s > 1/2$, by the Cauchy-Schwarz inequality, the true coefficients are absolutely summable. Specifically, there exists a constant $R < \infty$ such that 
$$
\|\thetaVec^*\|_1= \sum_{j=1}^\infty |j^s \theta_{f,j} j^{-s}|\le (\sum_{j=1}^\infty |j^s \theta_{f,j})|^2)^\frac{1}{2}(\sum_{j=1}^\infty j^{-2s})^\frac{1}{2} \le R
$$
Let $\|\psi_{kj}\|_\infty \le M$. This implies $q_\tau$ is deterministically bounded: $\|q_\tau\|_\infty \le MR := B$. This important observation will help us improve the online estimator in Section \ref{Sec:method}.

\subsection{Assumptions}
The two assumptions on the distribution of $X$ and $Y$  are considered below in this paper.

\begin{assumption} \label{ass:density_X}
The marginal distribution of $\X$ has a density $p_X(\x)$ strictly bounded away from zero and infinity: $0 < C_1 \le p_X(\x) \le C_2 < \infty$ for all $\x \in [0,1]^p$.
\end{assumption}

\begin{assumption} \label{ass:density_Y}
The conditional distribution of $Y$ given $\X=\x$ is absolutely continuous with density $p_{Y|\X}(y|\x)$. Furthermore, there exist constants $c_1, c_2 > 0$ such that $c_1 \le p_{Y|\X}(y|\x) \le c_2$ for all $\x \in [0,1]^p$ and all $y \in [-B, B]$.
\end{assumption}

We make several remarks regarding the assumptions. Assumption \ref{ass:density_X} imposes standard upper and lower bounds on the marginal density of the covariates. This condition is ubiquitous in nonparametric regression \citep{gyorfi2002distribution}, as it establishes the equivalence between the theoretical $L_2$ norm under the uniform Lebesgue measure (where our basis functions are orthogonal) and the $L_2(P_X)$ norm (which governs the actual prediction error), ensuring that the function can be consistently learned across its entire domain. Assumption \ref{ass:density_Y} is the fundamental requirement for quantile estimation. Unlike the squared loss, which is globally strongly convex, the pinball loss is only piecewise linear. In expectation, the local curvature (second derivative) of the pinball loss around the true quantile is precisely dictated by the conditional density $p_{Y|\X}$. The strict lower bound $c_1 > 0$ ensures that the expected risk is locally strongly convex, providing the necessary expected descent direction for the stochastic gradient updates to converge.  Crucially, it only requires the density to be bounded on the compact interval $[-B, B]$, not on the entire real line.

\section{Online quantile regression via sub-gradient descent}\label{Sec:method}
\subsection{Online Learning}\label{3}
To illustrate our online learning algorithm clearly, we first suppose the server receives the i.i.d. data  $(X_t,Y_t),t=1,2,3,\ldots$ in a sequance manner. Namely, at time $t$, we receive a new data point $(X_t,Y_t)$ and do not gain the access of the previous data. The online sub-gradient descent algorithm calculates the corresponding sub-gradient of coefficients $\theta_j, j=1,2,\ldots$ and then updates the current quantile function in the direction of the negative sub-gradient, \textit{without} storing this new observation. Next, we give the details of our online learning algorithm.

At iteration $t$, let $J_t$ be the truncation dimension. We define the $(1+pJ_t)$-dimensional basis vector:
\begin{equation}
    \PsiVec_t(\x) = \big(1, \psi_{11}(x^{(1)}), \dots, \psi_{pJ_t }(x^{(p)})\big)^\top.
\end{equation}
 The estimator maintains a coefficient vector $\thetaVec_t$. Upon receiving the data point $(\X_t, Y_t)$, we calculate the subgradient of the pinball loss: $g_t = (\tau - \I(Y_t \le \thetaVec_{t-1}^\top \PsiVec_t(\X_t)))\PsiVec_t(\X_t)$. Thus, we have
$$
    \tilde{\thetaVec}_t = \thetaVec_{t-1} + \gamma_t\cdot g_t.
$$

However, the online estimator $\tilde{\thetaVec}_t^T\PsiVec_t $ (a series) can be divergent as $t\to\infty$ and the expanded coefficients  for true quantile function satisfies $\|\thetaVec^*\|_1 \le R$.  We can not ensure $\| \tilde{\thetaVec}_t\|_1 \le R$ and thus the naive online estimator $\tilde{\thetaVec}_t^T\PsiVec_t $ is difficult to be statistically consistent to the true function $q_\tau$. Note that 
$$ 
\|\tilde{\thetaVec}_t^T\PsiVec_t\|_\infty\le \|\tilde{\thetaVec}_t\|_1 \|\PsiVec_t\|_\infty\le M\|\tilde{\thetaVec}_t\|_1.
$$
To solve above two problems, we thus make a second step where we project the updated coefficients onto the $\ell_1$ ball $\B(R) = \{\thetaVec : \|\thetaVec\|_1 \le R\}$ for some large $R>0$. Thus, the second update rule is:
\begin{equation}
    \thetaVec_t = \Pi_{\B(R)}(\tilde{\thetaVec}_t), \label{eq:proj_step}
\end{equation}
where  $\Pi_{\B(R)}(\tilde{\thetaVec}_t):= \arg\min_{\thetaVec\in \B(R)}\|\tilde{\thetaVec}_t-\thetaVec\|_2$. 

 According to above arguments, our functional estimator at time $t$ is $$\widehat{q}_t(\x) = \thetaVec_t^\top \PsiVec_t(\x).$$
Because $\|\thetaVec_t\|_1 \le R$ and $\|\PsiVec_t(\x)\|_\infty \le M$, we guarantee $\|\widehat{q}_t\|_\infty \le B$ deterministically. The projection $\Pi_{\B(R)}$ can be computed exactly in $O(J_t)$ time using soft-thresholding. The above estimation details are shown in Algorithm \ref{alg:pfsgd}. Since $J_t$ strictly increases as $t$ grows,  the vector $\PsiVec_t(\X_t)$ is slightly longer than $\thetaVec_{t-1}$. Thus in the step \textbf{Dimension Alignment}, we simply append zeros to the end of $\thetaVec_{t-1}$ (which implies the coefficients for the newly added high-frequency basis functions are initialized to $0$) so that the vector addition in equation \eqref{eq:grad_step} is mathematically well-defined.

\begin{algorithm}[H]
\caption{Projected Functional Stochastic Gradient Descent (P-FSGD) for Quantile Regression}
\label{alg:pfsgd}
\begin{algorithmic}[1]
\REQUIRE Target quantile $\tau \in (0,1)$, $\ell_1$-ball radius $R > 0$, sequence of step sizes $\{\gamma_t\}_{t=1}^n$, sequence of truncation dimensions $\{J_t\}_{t=1}^n$.
\STATE \textbf{Initialize:} $\thetaVec_0 \leftarrow \mathbf{0}$
\FOR{$t = 1, 2, 3,\dots$}
    \STATE Receive new streaming data point $(\X_t, Y_t)$.
    \STATE \textbf{Basis Evaluation:} Evaluate the basis vector $\PsiVec_t(\X_t) \in \R^{1+pJ_t}$.
    \STATE \textbf{Dimension Alignment:} If $J_t > J_{t-1}$, pad $\thetaVec_{t-1}$ with $p(J_t - J_{t-1})$ zeros to match the current dimension.
    \STATE \textbf{Prediction:} Compute the quantile estimate for the  current  data:
    \[
        \widehat{f}_{t-1}(\X_t) = \thetaVec_{t-1}^\top \PsiVec_t(\X_t)
    \]
    \STATE \textbf{Gradient Step:} Update the coefficients in the unconstrained space:
    \begin{equation} \label{eq:grad_step}
        \tilde{\thetaVec}_t = \thetaVec_{t-1} + \gamma_t\cdot g_t,
    \end{equation}
    where  $g_t = \PsiVec_t(\X_t)(\tau - \I(Y_t \le \widehat{f}_{t-1}(\X_t)))$ is the current subgardient at the new data $(\X_t, Y_t)$.\\
    \STATE \textbf{Projection Step:} Project the coefficients onto the $\ell_1$-ball $\B(R)$ to guarantee bounded coefficients evaluations:
    \begin{equation} \label{eq:proj_step}
        \thetaVec_t = \Pi_{\B(R)}(\tilde{\thetaVec}_t)
    \end{equation}
\RETURN the current coefficient vector $\thetaVec_t$.
\ENDFOR
\end{algorithmic}
\end{algorithm}

Next, Algorithm \ref{alg:l1sphere_proj} tells us how to make such $\ell_1$ projection in \eqref{eq:proj_step} in practice. This algorithm implements the well-known projection onto the simplex (absolute values) and then restores signs. Importantly, it runs in $O(J\ln J)$ time only due to sorting.

\begin{algorithm}
\caption{Projection $\mathbf{u}$ onto the $\ell_1$-ball with radius $R$}
\label{alg:l1sphere_proj}
\begin{algorithmic}[1]
\REQUIRE $\mathbf{u} = (u_1,\dots,u_J) \in \mathbb{R}^J$, with $\|\mathbf{u}\|_1 > R$.
\ENSURE $\mathbf{v} \in \mathbb{R}^J$ such that $\|\mathbf{v}\|_1 = R$ and $\|\mathbf{u}-\mathbf{v}\|_2$ is minimized.

\STATE Compute absolute values: $a_i = |u_i|$ for $i=1,\dots,J$.
\STATE Sort $a_i$ in non-increasing order: $a_{(1)} \ge a_{(2)} \ge \cdots \ge a_{(J)} \ge 0$.
\STATE Initialize $\rho = 0$, $sum = 0$, $\lambda = 0$.
\FOR{$j = 1$ to $J$}
    \STATE $sum = sum + a_{(j)}$
    \STATE $\lambda_{\text{temp}} = (sum - R) / j$
    \IF{$\lambda_{\text{temp}} \ge a_{(j+1)}$ \textbf{or} $j = J$}
        \STATE $\rho = j$
        \STATE $\lambda = \lambda_{\text{temp}}$
        \STATE \textbf{break}
    \ENDIF
\ENDFOR
\STATE Compute $t_i = \max(a_i - \lambda,\, 0)$ for $i=1,\dots,J$.
\STATE Recover signs: $v_i = \operatorname{sgn}(u_i) \cdot t_i$ (with $\operatorname{sgn}(0)=0$).
\RETURN $\mathbf{v}$
\end{algorithmic}
\end{algorithm}

Then, we aim to establish the minimax optimality of the proposed estimator. By Assumption \ref{ass:density_X}, it is sufficient to establish the expected Lebesgue $L_2$ error: $ \E \|\widehat{q}_t - q_\tau\|_{L_2}^2:= \E \int (\widehat{q}_t(\x) - q_\tau(\x))^2d\x $.

\begin{theorem} \label{thm:main1}
Suppose Assumptions \ref{ass:density_X} and \ref{ass:density_Y} hold. Let the step size be $\gamma_t = A/t$ and the truncation dimension be $J_t = \lceil t^{1/(2s+1)} \rceil$. For a sufficiently large constant $A$, the MSE of the P-FGD quantile estimator satisfies:
\begin{equation}
    \E \|\widehat{q}_t - q_\tau\|_{L_2}^2 = O\Big(t^{-\frac{2s}{2s+1}}\Big).
\end{equation}
\end{theorem}

\subsection{Mini-Batch Extension}
When $n_t$ i.i.d. samples $\{(\X_{t,i},Y_i)\}_{i=1}^{n_t}$ arrive at time $t$, we compute the averaged mini-batch gradient:
$$
G_t= \frac{1}{n_t}\sum_{i=1}^{n_t}  g_{t,i} 
$$
where    $g_{t,i} = \PsiVec_t(\X_{t,i})(\tau - \I(Y_{t,i} \le \widehat{f}_{t-1}(\X_{t,i})))$ and update the coefficients by
$$
\tilde{\thetaVec}_t = \thetaVec_{t-1} + \gamma_t\cdot G_t.
$$
Similar to previous arguments, we also apply the projection step:
$$
 \thetaVec_t = \Pi_{\B(R)}(\tilde{\thetaVec}_t).
$$

\begin{theorem} \label{thm:main2}
Suppose Assumptions \ref{ass:density_X} and \ref{ass:density_Y} hold. Let the step size be $\gamma_t =A\cdot n_t/N_t$ where $N_t = \sum_{i=1}^t n_i$ is the cumulative sample size at time $t$. When $n_t/N_t\to 0$ and the truncation dimension be $J_t = \lceil N_t^{1/(2s+1)} \rceil$, for a sufficiently large constant $A$, the MSE of the P-FGD quantile estimator satisfies:
\begin{equation}
    \E \|\widehat{q}_t - q_\tau\|_{L_2}^2 = O\Big(N_t^{-\frac{2s}{2s+1}}\Big).
\end{equation}
\end{theorem}

\begin{remark}
  An interesting finding is that our online estimator also achieves the optimal minin-max rate although those $N_t$ data points are trained in batches.
\end{remark}

\section{Ensemble online estimators}
In this section, we give a method to improve the estimator proposed before. Motivated by the idea of random forest in \cite{breiman2001random}, in each step we can perform the gradient descent only w.r.t. $S_t<J_t$ variables that are randomly selected  in the full coordinates of $\theta_t$. Let $\hat{q}_\tau^r$ be the such online estimator at time $t$, the final estimator is equal to $\E(\hat{q}_\tau^r)$ which is an ensemble estimator and could be achieved in practice by calculating the mean average of some independently generated $\hat{q}_\tau^r$. The theoretical analysis of  $\E(\hat{q}_\tau^r)$ is similar to Theorem \ref{thm:main2}.

\section{Conclusion}
We have introduced a computationally efficient P-FGD estimator for online nonparametric quantile regression without storing historical data. To ensure our online estimator be always well defined, we implemented an $\ell_1$-projection  without sacrificing the $O(J_t\ln J_t)$ algorithmic runtime. Using a novel Hilbert space projection argument, we  rigorously prove our online estimator achieves  minimax optimality. The further research could be studying the order of selected basis functions when a data streaming is given.

\section{Proofs}

\subsection{Proof of Theorem \ref{thm:main2}}
Recall at time $J_t$ we have the $(1+pJ_t)$-dimensional basis vector:
\begin{equation*}
    \PsiVec_t(\x) = \big(1, \psi_{11}(x^{(1)}), \dots, \psi_{pJ_t }(x^{(p)})\big)^\top,\quad \x\in [0,1]^p.
\end{equation*}
At time $t$, the  online update  is defined by
\begin{align}
  \tilde{\thetaVec}_t& = \thetaVec_{t-1} + \gamma_t\cdot G_t \label{h}\\
  \thetaVec_t &= \Pi_{\B(R)}(\tilde{\thetaVec}_t),\label{shdj}
\end{align}
where $\Pi_{\B(R)}$ is the projection operator and   $G_t = \frac{1}{n_t}\sum_{i=1}^{n_t}( \tau - \I(Y_{t,i} \le \thetaVec_{t-1}^\top \PsiVec_t(\X_{t,i})))\PsiVec_t(\X_{t,i})$. Define the truncated quantile function $\check{q}_{J_t}$  by
\begin{equation}\label{kjk}
   \check{q}_{J_t}:= \sum_{k=1}^{p}\sum_{j=1}^{J_t} \langle q_\tau(\x), \psi_{k,j }(\x^{(k)}) \rangle_{L^2}\psi_{k,j }
\end{equation}
and the intermediate estimator w.r.t. $\tilde{\thetaVec}_t$:
$
 \tilde{q}_{t}(\x):= \tilde{\thetaVec}_t^T  \PsiVec_t(\x).
$
For any function $g=\sum_{k=1}^{p}\sum_{j=1}^{J_t}a_{k,j}\psi_{k,j }$, we have $\|g\|_{L^2}^2=\sum_{k=1}^{p}\sum_{j=1}^{J_t}a_{k,j}^2$ due to the orthogonality of $\psi_{k,j }$. Thus, $\hat{q}_{t}$ can be regarded as the projection of $\tilde{q}_{t}$ onto the function space $\mathcal{F}:=\{g=\sum_{k=1}^{p}\sum_{j=1}^{J_t}a_{k,j}\psi_{k,j } :  \sum_{j=1}^{J_t}|a_{k,j}|\le R\}$. Since $\sum_{k=1}^{p}\sum_{j=1}^{J_t}|\langle q_\tau(\x), \psi_{k,j }(\x^{(k)}) \rangle_{L^2}|\le \|\theta^*\|_1\le R$, we have $\check{q}_{J_t}\in\mathcal{F}$. By \eqref{shdj} and Hilbert's projection theorem, we have 
\begin{align}
 \| \tilde{q}_{t}-  \check{q}_{J_t}\|_{L^2}^2 &= \| \tilde{q}_{t}- \hat{q}_{t}+\hat{q}_{t}- \check{q}_{J_t}\|_{L^2}^2 \nonumber\\
 & = \| \tilde{q}_{t}- \hat{q}_{t}\|_{L^2}^2+\|\hat{q}_{t}- \check{q}_{J_t}\|_{L^2}^2+  2\langle \tilde{q}_{t}-\hat{q}_{t},\hat{q}_{t}- \check{q}_{J_t} \rangle_{L^2}, \label{se}\\
 &\ge \|\hat{q}_{t}- \check{q}_{J_t}\|_{L^2}^2, \label{sdq}
\end{align}
where in \eqref{se} Hilbert's projection theorem tells that $\langle \tilde{q}_{t}-\hat{q}_{t},\hat{q}_{t}- \check{q}_{J_t} \rangle_{L^2}\ge 0$. Adding $\|q_{\tau}- \check{q}_{J_t}\|_{L^2}^2$ on  both sides of \eqref{sdq}, by the orthogonality of $\Psi_{k,j}$ it is known that
\begin{equation}\label{jklk}
  \|\hat{q}_{t}- q_{\tau}\|_{L^2}^2\le \| \tilde{q}_{J_t}-q_\tau\|_{L^2}^2\quad a.s..
\end{equation}
The above inequality is essential since it implies we only need to analyze  $\| \tilde{q}_{J_t}-q_\tau\|_{L^2}^2$ later.

 According to \eqref{h}, the following relationship holds:
 \begin{align}
  \E \| \tilde{q}_{J_t}-q_\tau\|_{L^2}^2 & = \E\|  (\thetaVec_{t-1} + \gamma_t\cdot G_t)^T\PsiVec_t-q_\tau\|_{L^2}^2\nonumber\\
   &=\E\| \hat{q}_{t-1}-q_\tau+ \gamma_t\cdot G_t ^T\PsiVec_t \|_{L^2}^2\nonumber\\
   &= \E\| \hat{q}_{t-1}-q_\tau\|_{L^2}^2+\underbrace{2\E(\langle  \hat{q}_{t-1}-q_\tau, \gamma_t\cdot G_t ^T\PsiVec_t  \rangle_{L^2})}_{II}+\underbrace{\E\|\gamma_t\cdot  G_t ^T\PsiVec_t \|_{L^2}^2}_{III}. \label{ui2}
 \end{align}
 In order to establish recursion related to $\E \| \hat{q}_{t}-q_\tau\|_{L^2}^2$ and $\E \| \hat{q}_{t-1}-q_\tau\|_{L^2}^2$, we next bound II and III respectively. \newline
 
 {\sc Upper bound of Part II.} By law of total expectations, it is known 
 \begin{align}
   II & = 2\E(\langle  \hat{q}_{t-1}-q_\tau, \gamma_t\cdot G_t^T\PsiVec_t  \rangle_{L^2})\nonumber\\
   &=  2\E(\langle  \hat{q}_{t-1}-q_\tau, \gamma_t\cdot (s_{t,1}\PsiVec_t(\X_{t,1}))^T\PsiVec_t  \rangle_{L^2}),\label{Yksd}
 \end{align}
 where $s_{t,1}:=\tau - \I(Y_{t,1} \le \thetaVec_{t-1}^\top \PsiVec_t(\X_{t,1}))$ denotes the outer subgradient of pinball loss.

 We use the theory of RKHS to analyze \eqref{Yksd}.  Define the subspace $\mathcal{G}$ in $L^2$ space by $\mathcal{G}_t:=\{g=\sum_{k=1}^{p}\sum_{j=1}^{J_t}a_{k,j}\psi_{k,j } :a_{k,j}\in\R   \}$. The reproducing kernel for the truncated space $\mathcal{G}_t$ is exactly $K_t(\mathbf{u}, \mathbf{v}): = \PsiVec_t(\mathbf{u})^\top \PsiVec_t(\mathbf{v})$. Since $L^2$ is a Hilbert space, we can find another infinite dimensional subspace $\mathcal{G}_t^\perp$ s.t. $\mathcal{G}_t^\perp\bigoplus\mathcal{G}_t=L^2$.  Later, we need an important property of $K_t$ below. For any $g\in \mathcal{G}_t$ and $g^\perp\in \mathcal{G}_t^\perp$, it is not difficult to see
 \begin{align*}
  &\langle g, K(\x,\cdot) \rangle = g(\x), \quad \forall \x\in [0,1]^p,\\
  &\langle g^\perp, K(\x,\cdot) \rangle =0, \quad \quad \forall \x\in [0,1]^p.
 \end{align*}
 According to this property of $K_t$, we can calculate the inner product in II as follows
\begin{equation}\label{1}
   \inner{\widehat{q}_{t-1} - q_\tau}{K_t(\X_{t,1}, \cdot)}_{L_2} = \widehat{q}_{t-1}(\X_{t,1}) - \check{q}_{J_t}(\X_{t,1}),
\end{equation}
where we have defined $\check{q}_{J_t}$ in \eqref{kjk}. On the other hand, recall $\X_t:=\{\X_{t,1},\ldots,\X_{t,n_t}\}$ and $Y_t:=\{Y_{t,1},\ldots,Y_{t,n_t}\}$ be the collection of data at time $t$. By the Mean Value Theorem on the conditional CDF, 
\begin{equation}\label{2}
  \E[s_{t,1} \mid \X_1,Y_1,\ldots,\X_{t-1},Y_{t-1},\X_t ] = - p_{Y|\X_{t,1}}(\xi_t) (\widehat{q}_{t-1}(\X_t) - q_\tau(\X_t)),
\end{equation}
where $\tau= \Pro(Y\le q_\tau(\X)\mid \X)\ a.s.$ is used. Note that $\xi_t$ depends on $\X_1,Y_1,\ldots,\X_{t-1},Y_{t-1},\X_{t,1}$.  Because $\|\widehat{q}_{t-1}\|_\infty\le \|\hat{\theta}_{t-1}\|_1\cdot\sup_{k,j}\|\Psi_{k,j}\|_\infty \le RM \le B$ and $\|q_\tau\|_\infty \le  B$ where $M=\sup_{k,j}\|\Psi_{k,j}\|_\infty$ and the constant $B$ is in the last paragraph of Section \ref{sec:basis}, we know $\xi_t \in [-B, B]$, and thus $c_1 \le p_{Y|\X_{t,1}}(\xi_t) \le c_2$ by Assumption \ref{ass:density_Y}.

Next, we establish the following novel lemma based on Hilbert space orthogonal projection by using \eqref{1} and \eqref{2}:

\begin{lemma} \label{lem:hilbert_trick}
Let $q(\x) = \widehat{q}_{t-1}(\x) - q_\tau(\x)$, and let $h_t = \widehat{q}_{t-1} - \check{q}_{J_t}$ be its projection onto the first $J_t$ basis functions. Then:
\begin{equation*}
    II  \le -\gamma_t (c_1 C_1) \|h_t\|_{L_2}^2 + \gamma_t \frac{c_2^2 C_2^2}{c_1 C_1} \|q_\tau - \check{q}_{J_t}\|_{L_2}^2.
\end{equation*}
\end{lemma}
\begin{proof}
Decompose $q = h_t + q^\perp$, where $q^\perp = \check{q}_{J_t} - q_\tau$ is the orthogonal approximation error. According to \eqref{Yksd}, \eqref{1}, \eqref{2} and the law of total expectation, 
\begin{align*}
    II &= 2\gamma_t\cdot \E[s_{t,1}(\widehat{q}_{t-1}(\X_t) - \check{q}_{J_t}(\X_t))]\\
    &=2\gamma_t\cdot \E[ (\widehat{q}_{t-1}(\X_t) - \check{q}_{J_t}(\X_t))\E(s_{t,1}  \mid \X_1,Y_1,\ldots,\X_{t-1},Y_{t-1},\X_t) ]\\
    &=-2\gamma_t\cdot \E( p_{Y|\X_{t,1}}(\xi_t) q(\X_t) h_t(\X_t) ) \\
    &= -2\gamma_t\left[ \E(p_{Y|\X_{t,1}}(\xi_t) h_t^2(\X_t)) + \E(p_{Y|\X_{t,1}}(\xi_t) h_t(\X_t)q^\perp(\X_t)) \right]\\
    &\le 2\gamma_t\left[- c_1 C_1 \|h_t\|_{L_2}^2 + c_2 C_2 \int |h_t(\x)| |q^\perp(\x)| d\x\right],
\end{align*}
where in the last line Assumption \ref{ass:density_X}\&\ref{ass:density_Y} are used. Applying the weighted AM-GM inequality $ab \le \frac{c_1 C_1}{2 c_2 C_2} a^2 + \frac{c_2 C_2}{2 c_1 C_1} b^2$ to the second term yields:
\begin{align*}
   - c_1 C_1 \|h_t\|_{L_2}^2 + c_2 C_2 \int |h_t(\x)| |q^\perp(\x)| d\x &\le -c_1 C_1 \|h_t\|_{L_2}^2 + \left( \frac{c_1 C_1}{2} \|h_t\|_{L_2}^2 + \frac{c_2^2 C_2^2}{2 c_1 C_1} \|q^\perp\|_{L_2}^2 \right) \\
    &= -\frac{c_1 C_1}{2} \|h_t\|_{L_2}^2 + \frac{c_2^2 C_2^2}{2 c_1 C_1} \|q^\perp\|_{L_2}^2.
\end{align*}
Multiplying by $2\gamma_t$ completes the proof.
\end{proof}

By Lemma \ref{lem:hilbert_trick} and orthogonality of $\Psi_{k,j}$, we now have the decomposition:
\begin{equation}\label{d2}
  II \le -\gamma_t (c_1 C_1) \underbrace{\E(\|\widehat{q}_{t-1}-q_\tau\|^2_{L^2})}_{MSE\ at\ time\ t-1} + \gamma_t (\frac{c_2^2 C_2^2}{c_1 C_1}+c_1C_1) \underbrace{\|q_\tau - \check{q}_{J_t}\|_{L_2}^2}_{approximation\ error}.
\end{equation}
We bound the  approximation error $\|q_\tau - \check{q}_{J_t}\|_{L_2}^2$ as follows. Let $q_\tau=\sum_{k=1}^{p}f_k=\sum_{k=1}^{p}\sum_{j=1}^{\infty}\theta^*_{k,j}\Psi_{k,j}$ where $f_k\in  \mathcal{W}_1(s, Q, \{\psi_j\})$. Then,
\begin{align}
  \|q_\tau - \check{q}_{J_t}\|_{L_2}^2 & = \int \left(\sum_{k=1}^{p}\sum_{j=1}^{J_t+1}\theta^*_{k,j}\Psi_{k,j}(\x)\right)^2 d\x\nonumber\\
  &= \sum_{k=1}^{p}\sum_{j=1}^{J_t+1}(\theta^*_{k,j})^2\nonumber\\
  &\le \sum_{k=1}^{p}\frac{1}{J_t^{2s}} \sum_{j=J_t+1}^{\infty}(\theta^*_{k,j})^2 j^{2s} \le \frac{\sum_{k=1}^{p}\|f_k\|_{soblev}^2}{J_t^{2s}}. \label{ds}
\end{align}
The combination of \eqref{d2} and \eqref{ds} gives that
\begin{equation}\label{23}
  II\le -\gamma_t (c_1 C_1)\cdot \E\|\widehat{q}_{t-1}-q_\tau\|^2_{L^2} + \gamma_t (\frac{c_2^2 C_2^2}{c_1 C_1}+c_1C_1) \frac{\sum_{k=1}^{p}\|f_k\|_{soblev}^2}{J_t^{2s}}.
\end{equation}
 
{\sc Upper bound of Part III.} Next, we consider Part III in \eqref{ui2}. Here, we use the decomposition below. Define 
 $$
 V_i:= \gamma_t  ( \tau - \I(Y_{t,i} \le \thetaVec_{t-1}^\top \PsiVec_t(\X_{t,i})))\PsiVec_t(\X_{t,i})^T\PsiVec_t.
 $$
 Then, it is known 
 $$
 G_t ^T\PsiVec_t= \frac{1}{n_t}\sum_{i=1}^{n_t} V_i.
 $$
 According to above equation, we can rewrite III as follows:
 \begin{align}
   III & = \E\left( \left\|\frac{1}{n_t}\sum_{i=1}^{n_t} V_i\right\|_{L^2}^2 \right) \nonumber\\
   &= \frac{1}{n_t}\cdot \underbrace{\E\left( \frac{1}{n_t}\sum_{i=1}^{n_t} \left\|V_i\right\|_{L^2}^2\right)}_{IV}+ \underbrace{\frac{1}{n_t^2} \E\left( \sum_{i\neq j} \langle V_i,V_j\rangle_{L^2}\right)}_{V}.\label{vs}
 \end{align}
The above equation implies we need bounding IV and V separately. 

Firstly, we consider Part IV in \eqref{vs}. Given $\theta_{t-1}$,  $V_1,\ldots, V_{n_t}$ are conditionally i.i.d.. Therefore,
\begin{align}
  IV & = \E(\left\|V_1\right\|_{L^2}^2) \nonumber\\
  &= \E \left( \|  \gamma_t  ( \tau - \I(Y_{t,1} \le \thetaVec_{t-1}^\top \PsiVec_t(\X_{t,1})))\PsiVec_t(\X_{t,1})^T\PsiVec_t\|_{L^2}^2  \right) \nonumber\\
  &\le  \gamma_t^2 \E \left( \|\PsiVec_t(\X_{t,1})^T\PsiVec_t\|_{L^2}^2 \right) \nonumber\\
  &= \gamma_t^2 \E \left( \int (\PsiVec_t(\X_{t,1})^T\PsiVec_t(\x))^2 d\x\right)  \nonumber\\
  &= \gamma_t^2 \int  \E[(\PsiVec_t(\X_{t,1})^T\PsiVec_t(\x))^2] d\x \nonumber\\
  &\le C_2\gamma_t^2 \int \int(\PsiVec_t(\x_1)^T\PsiVec_t(\x_2))^2d\x_1 d\x_2\nonumber\\
  &\le C_2\gamma_t^2\cdot M^2pJ_t, \label{audh3}
\end{align}
where in the last line we use the orthogonality of basis $\psi_{k,j }$ and $\sup_{k,j}|\psi_{k,j }|\le M$. 

Secondly, we consider Part V in \eqref{vs}.  Let $\mathcal{T}_{t-1}$ be the $\sigma$-algebra generated by $\X_1,Y_1,\ldots,\X_{t-1},$ $Y_{t-1}$. By the conditional independence of $V_i$ and $V_j$, 
\begin{align}
  V &=  \frac{1}{n_t^2} \E\left( \sum_{i\neq j} \langle V_i,V_j\rangle_{L^2}\right)= \frac{1}{n_t^2} \E\E\left( \sum_{i\neq j} \langle V_i,V_j\rangle_{L^2}\mid \mathcal{T}_{t-1}\right) \nonumber\\
  &= \frac{n_t^2-n_t}{n_t^2} \E\E\Big( \int \gamma_t  ( \tau - \I(Y_{t,i} \le \thetaVec_{t-1}^\top \PsiVec_t(\X_{t,i})))\PsiVec_t(\X_{t,i})^T\PsiVec_t(\x)\nonumber\\
  &\quad\quad\quad\quad\quad\cdot \gamma_t  ( \tau - \I(Y_{t,j} \le \thetaVec_{t-1}^\top \PsiVec_t(\X_{t,j})))\PsiVec_t(\X_{t,j})^T\PsiVec_t(\x)d\x\mid \mathcal{T}_{t-1}
   \Big)\nonumber\\
 &= \frac{n_t-1}{n_t}  \E\int\E\Big[ \gamma_t  ( \tau - \I(Y_{t,i} \le \thetaVec_{t-1}^\top \PsiVec_t(\X_{t,i})))\PsiVec_t(\X_{t,i})^T\PsiVec_t(\x)\mid \mathcal{T}_{t-1}\Big]\nonumber\\
  &\quad\quad\quad\quad\quad\cdot \E\Big[\gamma_t  ( \tau - \I(Y_{t,j} \le \thetaVec_{t-1}^\top \PsiVec_t(\X_{t,j})))\PsiVec_t(\X_{t,j})^T\PsiVec_t(\x) \mid \mathcal{T}_{t-1}\Big]d\x
  \nonumber\\
  &=  \frac{n_t-1}{n_t} \E \|\E(V_1\mid \mathcal{T}_{t-1})\|_{L^2}^2. \label{bks}
\end{align}
Thus, the left thing is to bound $\|\E(V_1\mid \mathcal{T}_{t-1})\|_{L^2}^2$. Let $\mathcal{H}_{t}$ be the $\sigma$-algebra generated by $\X_1,Y_1,\ldots,\X_{t-1},$ $Y_{t-1},\X_{t,1}$. According to the telescope property of conditional expectation,
\begin{align}
  \E(V_1\mid \mathcal{T}_{t-1}) & = \E(\gamma_t  ( \tau - \I(Y_{t,1} \le \thetaVec_{t-1}^\top \PsiVec_t(\X_{t,1})))\PsiVec_t(\X_{t,1})^T\PsiVec_t\mid \mathcal{T}_{t-1})\nonumber\\
  &=\E\Big[\E[\gamma_t  ( \tau - \I(Y_{t,1} \le \thetaVec_{t-1}^\top \PsiVec_t(\X_{t,1})))\PsiVec_t(\X_{t,1})^T\PsiVec_t\mid\mathcal{H}_{t}]\mid \mathcal{T}_{t-1}\Big]\nonumber\\
  &=\gamma_t\cdot\E[  ( \tau - F_{Y|\X_{t,1}}(\widehat{q}_{t-1}(\X_{t,1})))\PsiVec_t(\X_{t,1})^T\PsiVec_t\mid \mathcal{T}_{t-1}]\nonumber\\
    &=\gamma_t\cdot\underbrace{\E[  ( \tau - F_{Y|\X_{t,1}}(\widehat{q}_{t-1}(\X_{t,1})))\PsiVec_t(\X_{t,1})^T\mid \mathcal{T}_{t-1}]\cdot\PsiVec_t}_{VI},\nonumber
\end{align}
where the notation $\cdot$ in VI denotes the vector inner production. Now, we focus on Part VI. Given $\mathcal{T}_{t-1}$, construct a real function $\alpha(\x):= (\tau-F_{Y|\X=x}(\widehat{q}_{t-1}(\x))))p_\X(\x)$, which is $L^2$ integrable. Note that $\PsiVec_t$ is a class of orthogonal functions in $L^2[0,1]^d$ and
$$
 VI=\int \alpha(\x)\PsiVec_t(\x)d\x\cdot\PsiVec_t.
$$
The above observation shows that VI is the basis-expansion of $\alpha(\x),\x\in [0,1]^p$. Therefore, the application of  Bessel equality  gives that
\begin{align}
   \| \E(V_1\mid \mathcal{T}_{t-1})\|_{L^2}^2 &\le \gamma_t^2\cdot \|w\|_{L^2}^2\nonumber\\
   &= \gamma_t^2\cdot \|(\tau-F_{Y|\X=x}(\widehat{q}_{t-1}(\x))))p_\X(\x)\|_{L^2}^2 \nonumber\\
    &= \gamma_t^2C_2^2\cdot\|\tau-F_{Y|\X=x}(\widehat{q}_{t-1}(\x)))\|_{L^2}^2\nonumber\\
    &= \gamma_t^2C_2^2\cdot\|F_{Y|\X=x}(q_\tau(\x))-F_{Y|\X=x}(\widehat{q}_{t-1}(\x)))\|_{L^2}^2\nonumber\\
    &\le  \gamma_t^2C_2^2c_2\cdot \|q_\tau-\widehat{q}_{t-1}\|_{L^2}^2, \label{bks2}
\end{align}
where in the third line Assumption \ref{ass:density_X} is applied and in the last line Lagrange mean value theorem  and Assumption \ref{ass:density_Y} are applied. The combination of \eqref{bks} and \eqref{bks2} gives us
\begin{equation}\label{bsajd}
  V\le  \gamma_t^2C_2^2c_2\cdot \E\|q_\tau-\widehat{q}_{t-1}\|_{L^2}^2.
\end{equation}

In conclusion, based on \eqref{vs}, \eqref{audh3} and \eqref{bsajd},  we have
\begin{equation}\label{Yjhdsa2}
  III\le C_2 M^2p\cdot\frac{J_t}{n_t}\gamma_t^2+  \gamma_t^2C_2^2c_2\cdot \E\|q_\tau-\widehat{q}_{t-1}\|_{L^2}^2.
\end{equation}

At this step, combine inequalities \eqref{jklk}, \eqref{ui2}, \eqref{23} and \eqref{Yjhdsa2}. We get the final recursion:
\begin{align}
  \E\|q_\tau-\widehat{q}_{t}\|_{L^2}^2 & \le \left(1-\gamma_t (c_1 C_1)+ \gamma_t^2C_2^2c_2\right)\cdot   \E\|q_\tau-\widehat{q}_{t}\|_{L^2}^2 \nonumber\\
  &\quad\quad+\gamma_t (\frac{c_2^2 C_2^2}{c_1 C_1}+c_1C_1) \frac{\sum_{k=1}^{p}\|f_k\|_{soblev}^2}{J_t^{2s}}+C_2 M^2p\cdot\frac{J_t}{n_t}\gamma_t^2, \label{nS}
\end{align}
which holds for all $t=1,2,3,\ldots$.\\

{\sc Solve the recursion \eqref{nS}.} We first simplify some notations  in \eqref{nS}:
\begin{equation}\label{sh2bja2}
  e_t\le(1-2w_1\gamma_t+w_2\gamma_t^2)e_{t-1}+\gamma_t\frac{w_3}{J_t^{2s}}+w_4\cdot \frac{J_t}{n_t}\gamma_t^2,
\end{equation}
where $e_t:=  \E\|q_\tau-\widehat{q}_{t}\|_{L^2}^2 $ and those coefficients $w_i, i=1,2,3,4$ can be checked in \eqref{nS}. Since $\gamma_t\to 0$, we have $\gamma_t^2=o(\gamma_t)$ and \eqref{sh2bja2} can be further simplified by
\begin{equation}\label{sh2bja3}
  e_t\le(1-w_1\gamma_t)e_{t-1}+\gamma_t\frac{w_3}{J_t^{2s}}+w_4\cdot \frac{J_t}{n_t}\gamma_t^2,
\end{equation}
where $t=t_0,t_0+1,\ldots$ and $t_0$ is a \textbf{fixed integer}. 

\begin{lemma}\label{lemm:recursion}
Suppose for $t \ge t_0$, we have $e_t \le a_t e_{t-1} + b_t$ with $a_t, b_t, e_t > 0$. Then for any $t \ge t_0$,
\[
e_t \le e_{t_0-1} \prod_{k=t_0}^{t} a_k + \sum_{k=t_0}^{t-1} \left( \prod_{j=k+1}^{t} a_j \right) b_k + b_t.
\]
\end{lemma}

\begin{proof}
This result can be proved by mathematical induction. It can be checked the inequality holds when $t = t_0$. Now, assume this result holds for $t\ge t_0$. Then,
\begin{align*}
e_{t+1} &\le a_{t+1} e_t + b_{t+1} \\
&\le a_{t+1} \left( e_{t_0-1} \prod_{k=t_0}^{t} a_k + \sum_{k=t_0}^{t-1} \left( \prod_{j=k+1}^{t} a_j \right) b_k + b_t \right) + b_{t+1} \\
&= e_{t_0-1} \prod_{k=t_0}^{t+1} a_k + \sum_{k=t_0}^{t} \left( \prod_{j=k+1}^{t+1} a_j \right) b_k + b_{t+1},
\end{align*}
which completes the argument for $t-1$.
\end{proof}

According to \eqref{sh2bja3} and  Lemma \ref{lemm:recursion},  it holds
\begin{equation}\label{dsbnm}
e_t\le e_{t_0-1} \underbrace{ \prod_{k=t_0}^{t} (1-w_1\gamma_k)}_{VII}+ \underbrace{\sum_{k=t_0}^{t-1} \left( \prod_{j=k+1}^{t} (1-w_1\gamma_j)\right) \left(\gamma_k\frac{w_3}{J_k^{2s}}+w_4\cdot \frac{J_k}{n_k}\gamma_k^2\right)}_{VIII} + \underbrace{\gamma_t\frac{w_3}{J_t^{2s}}+w_4\cdot \frac{J_t}{n_t}\gamma_t^2}_{VX}.
\end{equation}

Firstly, we analyze VII. Actually, 
\begin{align*}
  VII & = \prod_{k=t_0}^{t} (1-w_1\gamma_k)=\prod_{k=t_0}^{t} (1-Aw_1\frac{n_t}{N_t})\le \prod_{k=t_0}^{t} (1-\frac{n_t}{N_t})=\frac{N_{t_0-1}}{N_t},
\end{align*}
where we use $N_t=n_t+N_{t-1}$ for each $t=1,2,3,\ldots$ and choose $Aw_1>1$. Secondly, we bound VIII. In this case, 
\begin{equation}\label{sdgfh}
  \gamma_k\frac{w_3}{J_k^{2s}}+w_4\cdot \frac{J_k}{n_k}\gamma_k^2\asymp n_k\cdot N_k^{-1-\frac{2s}{2s+1}}.
\end{equation}
Thus, according to above analysis of VII,
$$
  VIII\le \sum_{k=t_0}^{t-1} N_k^{-1-\frac{2s}{2s+1}} \frac{N_k}{N_{k+1}}\cdots\frac{N_{t-2}}{N_{t-1}} n_k=\frac{1}{N_{t-1}}\sum_{k=t_0}^{t-1} N_k^{-\frac{2s}{2s+1}}n_k.
$$
Note that the real function $v\in\R^+ \to v^{-\frac{2s}{2s+1}}$ is decreasing and $N_t=n_t+N_{t-1}$. We further have
$$
VIII\le \frac{1}{N_{t-1}}\int_{1}^{N_{t-1}} v^{-\frac{2s}{2s+1}}dv\le (2s+1)N_{t-1}^{\frac{1}{2s+1}}N_{t-1}^{-1}=(2s+1)N_{t-1}^{-\frac{2s}{2s+1}}.
$$
Since $n_t/N_t\to 0$, we assume without loss of generality that $n_t/N_t\le \frac{1}{2}$ for all $t\ge t_0$. Therefore, $N_{t-1}/N_t=1-n_t/N_t\ge \frac{1}{2}$ and we finally have
$$
VIII\le (2s+1)2^{\frac{2s}{2s+1}} N_{t}^{-\frac{2s}{2s+1}}.
$$
And by \eqref{sdgfh}, we know Part VX satisfies $VX\le c N_t^{-\frac{2s}{2s+1}}$ for some universal constant $c>0$. In conclusion, by \eqref{dsbnm} and above arguments, it holds
$$
e_t=O\Big(N_t^{-\frac{2s}{2s+1}}\Big),
$$
which completes the proof. \hfill\(\Box\)\\

\bibliographystyle{apalike}

\bibliography{ref}

\end{document}